\documentclass{article}

\PassOptionsToPackage{numbers}{natbib}
\usepackage{amsmath}
\usepackage{amsthm}
\usepackage{amssymb}
\usepackage{amsfonts}
\newtheorem{lemma}{Lemma}
\newtheorem{corollary}{Corollary}
\newtheorem{proposition}{Proposition}

\usepackage{graphicx}
\usepackage{diagbox}
\usepackage{wrapfig}
\usepackage{hyperref}

\usepackage{subfigure}
\DeclareMathAlphabet\mathbfcal{OMS}{cmsy}{b}{n}
     \usepackage[preprint]{neurips_2021}

\usepackage[utf8]{inputenc} 
\usepackage[T1]{fontenc}    
\usepackage{hyperref}       
\usepackage{url}            
\usepackage{booktabs}       
\usepackage{amsfonts}       
\usepackage{nicefrac}       
\usepackage{microtype}      
\usepackage{xcolor}         

\title{Interferometric Graph Transform for Community Labeling}

\author{
Nathan Grinsztajn$^*$  \\
Inria, Univ. Lille, CNRS\\
Lille, France\\
\texttt{nathan.grinsztajn@inria.fr} \\
\And
Louis Leconte\thanks{Equal contributions.}\\
LIP6, Sorbonne University\\
CMAP, Ecole Polytechnique, France\\
\texttt{louis.leconte@ens-paris-saclay.fr}\\
\And
Philippe Preux \\
Inria, Univ. Lille, CNRS\\
Lille, France\\
\And
Edouard Oyallon \\
CNRS, LIP6, Sorbonne University\\
Paris, France}

\begin{document}

\maketitle
\begin{abstract}
We present a new approach for learning unsupervised node representations in community graphs. We significantly extend the Interferometric Graph Transform (IGT) to community labeling: this non-linear operator iteratively extracts features that take advantage of the graph topology through demodulation operations. An unsupervised feature extraction step cascades modulus non-linearity with linear operators that aim at building relevant invariants for community labeling. Via a simplified model, we show that the IGT concentrates around the E-IGT: those two representations are related through some ergodicity properties.  Experiments on community labeling tasks show that this unsupervised representation achieves performances at the level of the state of the art on the standard and challenging datasets Cora, Citeseer, Pubmed and WikiCS.
\end{abstract}

\section{Introduction}\label{intro}
Graph Convolutional  Networks (GCNs) \cite{kipf2016semi} are now the state of the art for solving many supervised (using labeled nodes) and semi-supervised (using unlabeled nodes during training) graph tasks, such  as nodes or community labeling. They consist in a cascade of layers that progressively average node representations, while maintaining discriminative properties through supervision. In this work, we are mainly interested in the principles that allow such models to outperform other baselines: we propose a specific class of GCNs, which is unsupervised, interpretable, with several theoretical guarantees while obtaining good accuracies on standard datasets.

One of the reasons why GCNs lack interpretability is because no training objective is assigned to a specific layer except the final one: end-to-end training makes their analysis difficult \cite{oyallon2017building}. They also tend to oversmooth graph representations \cite{yang2016revisiting}, because applying successively an averaging operator leads to smoother representations. Also, the reason of their success is in general unclear \cite{li2018deeper}. In this work, we propose to introduce a novel architecture which, by design, will address those issues. Our model can be interpreted through the lens of Stochastic Block Models (SBMs) \cite{holland1983stochastic} which are standard, yet are not originally designed to analyze graph attributes through representation learning.

For example, several works \cite{keriven2020sparse, abbe2017community} prove that a Laplacian matrix concentrates around a low-rank expected Laplacian matrix, via simplified models like a SBM \cite{cohen2020power}. In the context of community detection, it is natural to assume that the intra-class, inter-class connectivity and feature distributions of a random graph are ruled by an SBM. To our knowledge, this work is the first to make a clear connection with those unsupervised models and the  self-supervised deep GCNs which solve datasets like Cora, Citeseer,  Pubmed, or WikiCS.

Our model is driven by ideas from the Graph Signal Processing \cite{hammond2011wavelets} community and based on the Interferometric Graph Transform \cite{oyallon2020interferometric}, a class of models mainly inspired by the (Euclidean) Scattering Transform \cite{mallat2012group}. The IGT aims at learning unsupervised (not using node labels at the representation learning stage), self-supervised representations that correspond to a cascade of isometric layer and modulus non-linearity, whose goal is to obtain a form of demodulation \cite{oyallon2018compressing} that will lead to smoother but discriminative representation, in the particular case of community labeling. Smooth means here, by analogy with Signal Processing \cite{mallat1999wavelet}, that the signal is in the low-frequency domain, which corresponds to a quite lower dimensional space if the spectral decay is fast enough: this is for instance the case with a standard Laplacian~\cite{grinsztajn2021lowrank} or a low-rank SBM adjacency matrix~\cite{loukas2018spectrally}. Here,  the degree of invariance of a given representation is thus characterized by the smoothness of the signal.

Our main contribution is to introduce a simplified framework that allows to analyze node labeling tasks based on a non-linear model, via concentration bounds and which is  numerically validated. Our other contributions are as follows. First, we introduce a novel  graph representation for community labeling, which doesn't involve community labels. It consists in a cascade of linear isometry, band-pass filtering, pointwise absolute value non-linearity. We refer to it as an Interferometric Graph Transform (IGT) (for community labeling), and we show that under standard assumptions on the graph of our interest, a single realization of our representation concentrates around the Expected Interferometric Graph Transform (E-IGT), which can be defined at the node level without incorporating any graph knowledge. We also introduce a novel notion of localized low-pass filter, whose invariance can be adjusted to a specific task. Second, we study the behavior of this representation under an SBM model: with our model and thanks to the structure of the IGT, we are able to demonstrate theoretically that IGT features accumulate around the corresponding E-IGT. We further show that the architecture design of IGTs allows to outperform GCNs in a synthetic setting, which is consistent with our theoretical findings. Finally, we show that this semi-supervised and unsupervised representation  is numerically competitive with supervised representations on standard community labeling datasets like Cora, Citeseer, Pubmed and WikiCS.

Our paper is organized as follows. First, we define the IGT in Sec.\@ \ref{sec:IGT} and study its basic properties.  Sec.\@ \ref{sec:expected-IGT} defines the E-IGT and bounds its distance from the IGT. Then, we discuss our model in the context of a SBM in Sec.\@ \ref{sec:model} and we explain our optimization procedure in Sec.\@ \ref{sec:optim}. Finally, Sec.\@ \ref{sec:xp} corresponds to our numerical results. Our source can be found at \href{https://github.com/nathangrinsztajn/igt-community-detection}{https://github.com/nathangrinsztajn/igt-community-detection} and all proofs of our results can be found in the Appendix.

\vspace{-5pt}
\section{Related Work}
\vspace{-5pt}
We now discuss a very related line of work, namely the IGT \cite{oyallon2020interferometric}, which takes source in several conceptual ideas from the Scattering Transform \cite{mallat2012group}. Both consist in a cascade of unitary transform, absolute value non-linearity and linear averaging, except that the Euclidean structure is neatly exploited via Wavelets Transforms for complex classification tasks in the case of the standard Scattering Transform \cite{bruna2013invariant,oyallon2015deep,7324385,oyallon2018compressing}, whereas this structure is implicitly used in the case of IGT. In particular, similarly to a Scattering Transform, an IGT aims at projecting the feature representation in a lower dimensional space (low-frequency space) while being discriminative: the main principle is to employ linear operators, which combined with a modulus non-linearity, leads to a demodulation effect. In our case however, this linear operator is learned. The IGT for community labeling is rather different from standard IGT: first, \cite{oyallon2020interferometric} is not amenable to node labeling because it doesn't preserve node localization, contrary to ours. Second, we do not rely on the Laplacian spectrum explicitely contrary to \cite{gama2020stability,oyallon2020interferometric}. Third, the community experiments of \cite{gama2020stability,oyallon2020interferometric} are rather the classification of a diffusion process than a node labeling task. This is also similar to the Expected Scattering Transform \cite{mallat2013deep}, yet it is applied in a rather different context for reducing data variance, in order to shed lights on standard Deep Neural Networks. Our E-IGT and the Expected-Scattering have a very close architecture, however the linear operators are obtained with rather different criteria (e.g., ours are obtained from a concave procedure rather than convex) and goals (e.g., preserving energy, whereas we try to reduce it). Note however there is no equivalent of the E-IGT for other context that community detection or labeling, which is another major difference with \cite{oyallon2020interferometric}. In addition, our Prop. \ref{prop:boundigt} is new compared to similar results of \cite{mallat2013deep}. Thus while having similar architectures, those works have quite different outcomes and objectives.

Another line of works corresponds to the Graph Scattering Transform \cite{gama2020stability,gao2019geometric,ioannidis2020efficient}, which proposes to employ a cascade of Wavelet Transforms that respects the graph structure \cite{hammond2011wavelets}. Yet, the principles that allow good generalization of those representations are unclear and they have only been tested until now on small datasets. Furthermore, this paper extends all those works by proposing an architecture and theoretical principles which are specific to the task of community labeling. A last related line of work corresponds to the hybrid Scattering-GCNs \cite{min2020scattering}, which combines a GCN with the inner representation of a Scattering Transform on Graphs, yet they employ massive supervision to refine the weights of their architecture, which we do not do.

The architecture of an IGT model for community labeling takes also inspiration from  Graph Convolutional Networks (GCNs) \cite{kipf2016semi, bronstein2017geometric}. They are a cascade of linear operators and ReLU non-linearities whose each layer is locally averaged along local nodes. Due to this averaging, GCNs exhibit two undesirable properties: first, the oversmoothing phenomenon \cite{li2018deeper}, which makes learning of high-frequencies features difficult; second, the training of deeper GCNs is harder \cite{huang2020tackling} because much information has been discarded by those averaging steps. Other types of Graph Neural Networks succeeded in approximating message-passing methods \cite{chen2018supervised}, or have worked on the spatial domain such as 
Spectral GCNs \cite{bruna2013spectral}, and Chebynet \cite{defferrard2016convolutional}. In our work, we solely use a well chosen averaging for separating high-frequencies and low-frequencies without using any other extra-structure, which makes our method more generic than those approaches, without using supervision at all.

We further note that theoretical works often address the problem of estimating the expected Laplacian under SBM assumptions \cite{keriven2020sparse, abbe2017community, le2018concentration}.  However up to our knowledge, none of those works is applied in a semi-supervised context and they aim at discovering communities rather than estimating communities from a small subset of labels. Moreover, the model remains mostly linear (e.g. based on the spectrum of the adjacency matrix). Here, our representation is non-linear and amenable for learning with a supervised classifier. We also note that several theoretical results have allowed to obtain approximation or stability guarantees for GCNs \cite{ruiz2020graph, bruna2013invariant, keriven2019universal}: our work follows those lines and analyzes a specific type of GCN through the lens of Graph Signal Processing theory \cite{hammond2011wavelets}.
\vspace{-5pt}
\section{Framework} 
\vspace{-5pt}
\paragraph*{Notations.}
For a matrix $X$, we write $\Vert X\Vert^2=\text{Tr}(X^TX)=\sum_{i,j}X_{i,j}^2$ its Frobenius-norm and for an operator $L$ (acting on $X$), we might consider the related operator norm $\Vert L\Vert\triangleq \sup_{\Vert X\Vert \leq 1}\Vert LX\Vert $. The norm of the concatenation $\{B,C\}$ of two operators $B,C$ is $\Vert \{B,C\}\Vert^2=\Vert B\Vert^2+\Vert C\Vert^2$ and this definition can be extended naturally to more than two operators. Note also that we use a different calligraphy between quantities related to the graph (e.g., adjacency matrix $\mathbfcal{A}$) and operators (e.g., averaging matrix $A$). We write $A \preccurlyeq B$ if $B-A$ is a symmetric positive matrix.  Here, $a_n\sim b_n$ means that $\exists \alpha>0,\beta>0:\alpha |a_n|\leq |b_n|\leq \beta |b_n|$ and $a_n=\mathcal{O}(b_n)$ means $\exists \alpha>0:|a_n|\leq \alpha|b_n|$.

\subsection{Definition of IGT}\label{defIGT}
\label{sec:IGT}

Our initial graph data are node features $X\in\mathbb{R}^{n\times P}$ obtained from a graph with $n$ nodes and unormalized adjacency matrix $\mathbfcal{A}$. We then write $ \mathbfcal{A}_{\text{norm}}$ the normalized adjacency matrix  with self-connexion, as introduced by \cite{kipf2016semi}. We note that $\mathbfcal{A}_{\text{norm}}$ satisfies $0 \preccurlyeq \mathbfcal{A}_{\text{norm}} \preccurlyeq I$ and has positive entries. In Graph Signal Processing \cite{hammond2011wavelets}, those properties allow to interpret $\mathbfcal{A}_{\text{norm}}$ as an averaging operator. It means that applying $\mathbfcal{A}_{\text{norm}}$ to $X$ leads to a linear representation $\mathbfcal{A}_{\text{norm}}X$ which is smoother than $X$ because $\mathbfcal{A}_{\text{norm}}$ projects
the data in a subspace ruled by the topology (or connectivity) of a given community~\cite{gama2020stability}. The degree of smoothness can be adjusted to a given task simply by considering:
\begin{equation}
\label{smoothness}
    A_J \triangleq \mathbfcal{A}_{\text{norm}}^J\,.
\end{equation}
This step is analogeous to the rescaling of a low-pass filter in Signal Processing \cite{mallat1999wavelet}, and $A_J$ satisfies:
\begin{lemma}\label{lemma:positive}
If $0 \preccurlyeq \mathbfcal{A}_{\text{norm}} \preccurlyeq I$ and $\mathbfcal{A}_{\text{norm}}$ has positive entries, then for any $J\in \mathbb{N}$, $A_J$ has positive entry and satisfies also $0 \preccurlyeq A_J \preccurlyeq I$.
\end{lemma}
Applying solely $A_J$ leads to a loss of information that we propose to recover via $I-A_J$. This allows to separate low and high-frequencies of the graph in two channels, as expressed by the next lemma:
\begin{lemma}\label{proj}
If $0 \preccurlyeq A \preccurlyeq I$, then $\Vert AX\Vert^2+\Vert (I-A)X\Vert^2\leq \Vert X\Vert^2$ with equality iff $A^2=A$.
\end{lemma}


Yet, contrary to $A_JX$, $(I-A_J)X$ is not smooth and thus, it might not be amenable for learning because community structures might not be preserved. Furthermore, a linear classifier will not be sensitive to the linear representation $\{A_JX,(I-A_J)X\}$. Similarly to \cite{oyallon2020interferometric}, we propose to apply an absolute value $|.|$ point-wise non-linearity to our representations. Section \ref{sec:optim} will explain how to estimate isometries $\{W_n\}$, which combined with a modulus, will smooth the signal envelope while preserving signal energy. We now formally describe our architecture and we consider $\{W_n\}$ a collection of isometries, that we progressively apply to an input signal representation $U_0\triangleq X$ via: 
\begin{equation}
\label{eq:u}
    U_{n+1} \triangleq |(I-A_J)U_nW_n|\,,
\end{equation}
and we introduce the IGT representation of order $N\in\mathbb{N}$ with averaging scale $J\in\mathbb{N}$ defined by:
\begin{equation}
    S^N_JX\triangleq\{A_JU_0,...,A_JU_N\}\,.
\end{equation}
Fig. \ref{icml-historical} depicts our architecture. The following explains that $S_J^N$ is non-expansive, thus  stable to noise:

\begin{proposition}
\label{prop:lip}
For $N\in\mathbb{N}$, $S_J^NX$ is 1-Lipschitz leading to:
\begin{equation}
\Vert S_J^NX-S_J^NY\Vert\leq\Vert X-Y\Vert\, \text{and, } \Vert S_J^NX\Vert\leq\Vert X\Vert\,.
\end{equation}
\end{proposition}
\begin{wrapfigure}{r}{0.35\textwidth}
\vspace{-18pt}
\begin{center}\centerline{\includegraphics[trim={27.5cm 15cm 22cm 4.5cm}, clip, width=0.4\columnwidth]{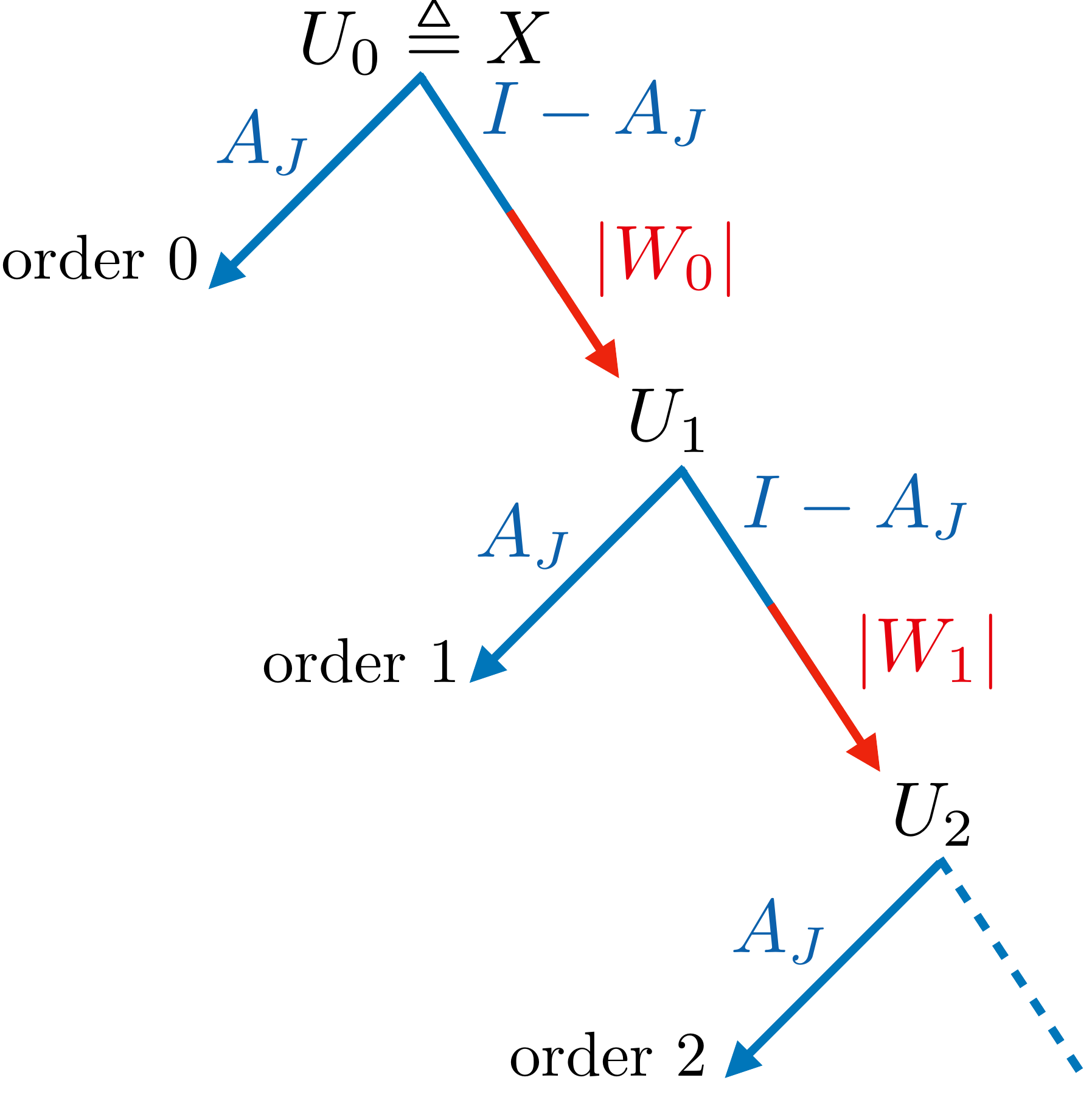}}\caption{We illustrate our model for $N=2$. Low and high frequencies are separated (blue) and then the high frequencies are demodulated (red) via an isometry and a non-linear point wise absolute value, and then propagated to the next layer.}\vspace{-22pt}\label{icml-historical}\end{center}\end{wrapfigure}

The next section will describe the E-IGT, which was introduced as the Expected Scattering \cite{mallat2013deep}, but in a rather different context: we will show under simplifying assumptions that an IGT for community labeling concentrates around the E-IGT.

\subsection{Definition of the Expected-IGT (E-IGT)}
\label{sec:expected-IGT}
Similarly to the previous section, for an input signal $\bar U_0\triangleq X$, we consider the following recursion, introduced in \cite{mallat2013deep}:
\begin{equation}
    \bar U_{n+1} \triangleq |(\bar U_n-\mathbb{E}\bar U_n)W_n|\,,
\end{equation}
which leads to the E-IGT  \footnote{We rename it here because we use rather different principles to obtain the $\{W_0...,W_{N-1}\}$ compared to the original Scattering.} of order $N$ defined by:
\begin{equation}
    \bar S_N\triangleq\{\mathbb{E}\bar U_0,...,\mathbb{E}\bar U_N\}\,.
\end{equation}Similarly to Prop. \ref{prop:lip}, we prove the following stability result:
\begin{proposition}
\label{prop:lip2}
For $N\in\mathbb{N}$, $\bar S^NX$ is 1-Lipschitz, meaning that:
\begin{equation}
\Vert \bar S^NX- \bar S^NY\Vert^2\leq\mathbb{E}[\Vert X-Y\Vert^2]\,,
\end{equation}
and furthermore:
\begin{equation}
\Vert \bar S^NX\Vert^2 \leq\mathbb{E}[\Vert X\Vert^2]\,.
\end{equation}
\end{proposition}
\begin{proof}
Indeed, \cite{mallat2013deep} have proven this for the columns  of $X$.
\end{proof}
Note that this represention is also more amenable to standard supervised classifiers such as SVMs because no operation mixing nodes is involved. Prop. \ref{prop:lip2} highlights the fact that the E-IGT is non-expansive, and \cite{waldspurger2017wavelet} shows that this allows to discriminate the attributes of the distribution of $X$. However, it is difficult in general to estimate the E-IGT because one does not know the distribution of a given node and it is difficult to estimate it from a single realization as there is a clear curse of dimensionality. However, we will show that $S_J^N$ will be very similar to $\bar S^N$ under standard assumptions on communities. We now state the following proposition, which allows to quantify the distance between an IGT and its E-IGT:


\begin{proposition}\label{prop:boundigt}
For any $X, N, J$, we get:
\begin{equation}
  \Vert S_J^NX-\bar S^N X\Vert\leq   \sqrt{2} \sum_{m=0}^{N}\Vert  (A_J-\mathbb{E})\bar U_m\Vert \,.
\end{equation}
\end{proposition}

The proof of this proposition can be found in the Appendix: it fully uses the tree structure of Fig. \ref{icml-historical}, in order to obtain tighter bounds than \cite{mallat2013deep}, as it allows $N$ to be of arbitrary size without diverging. We  now bound the distance between the IGT and the E-IGT:
\begin{corollary}\label{prop:concentration}
For $N\in \mathbb{N}$, we have:
\begin{equation}
\sup_{\mathbb{E}\Vert X\Vert\leq 1}\mathbb{E}[\Vert S_J^{N}X-  \bar S^{N}X\Vert] \leq 2^{N+2} \sup_{\mathbb{E}\Vert X\Vert\leq 1}\mathbb{E}[\Vert A_JX-\mathbb{E}X\Vert]\,.\label{ergo}
\end{equation}
\end{corollary}
\begin{proof}
The next Lemma combined with the norm homogeneity allows to conclude with Prop.\@ \ref{prop:boundigt}.
\end{proof}
\begin{lemma}\label{bounded}
If $\Vert X\Vert\leq 1$, then $\Vert\bar U_n\Vert\leq 2^n$, with $\bar U_0=X$. Also, if $\mathbb{E}[\Vert X\Vert]\leq 1$, then $\mathbb{E}[\Vert \bar U_n\Vert]\leq 2^n$
\end{lemma}
\begin{proof}
This is true for $n=0$, and then by induction, since isometry preserves the $\ell^2$-norm: $\Vert\bar U_{n+1}\Vert\leq \Vert\bar U_n\Vert+\Vert \mathbb{E}\bar U_n\Vert\leq \Vert\bar U_n\Vert+ \mathbb{E}\Vert\bar U_n\Vert\leq 2^{n+1}$. The proof is similar for the second part.
\end{proof}

The right term of Eq. \ref{ergo} measures the ergodicity properties of a given $A_J$. For instance, in the case of images, a stationary assumption on $X$ implies that $A_Jf(X)\approx\mathbb{E}f(X)$ for all measurable $f$, which is the case for instance for textures \cite{mallat1999wavelet}. The following proposition shows that in case of exact ergodicity, the two representations have bounded moments of order 2:
\begin{proposition}
If $\mathbb{E}[A_JX]=\mathbb{E}X$, and if $X$ has variance $\sigma^2=\mathbb{E}\Vert X \Vert^2- \Vert \mathbb{E}X \Vert^2$, then:
\begin{align}
    \mathbb{E}[\Vert S_J^NX-\bar S^NX\Vert^2 ] & \leq  2 \sigma^2\,.
\end{align}

\end{proposition}

\subsection{Graph model and concentration bounds}
\label{sec:model}

In this subsection, we propose to demonstrate novel bounds which improve  the upper bound obtained at Prop.\@ \ref{prop:concentration} by introducing a Stochastic Block Model \cite{holland1983stochastic}. We will show that the IGT features of a given community concentrates around the E-IGT feature of this community: IGT features are thus more amenable to be linearly separable.  Recall from Sec.\@ \ref{defIGT} that $A_1=\mathbfcal{A}_{\text{norm}}$, thus we note that for some $m>0$, via the triangular inequality we get:
\[
    \Vert A_1\bar U_m-\mathbb{E}\bar U_m\Vert=\Vert \mathbfcal{A}_{\text{norm}}\bar U_m-\mathbb{E}\bar U_m\Vert 
    \leq \Vert (\mathbfcal{A}_{\text{norm}}-\mathbb{E}[\mathbfcal{A}]_{\text{norm}})\bar U_m\Vert + \Vert \mathbb{E}[\mathbfcal{A}]_{\text{norm}}\bar U_m-\mathbb{E}[\bar U_m]\Vert\,.
\]
Now, the left term can be upper bounded as:
\begin{equation}
    \Vert (\mathbfcal{A}_{\text{norm}}-\mathbb{E}[\mathbfcal{A}_{\text{norm}}])\bar U_m\Vert\leq  \Vert \mathbfcal{A}_{\text{norm}}-\mathbb{E}[\mathbfcal{A}_{\text{norm}}]\Vert \Vert \bar U_m\Vert\,.
\end{equation}
 For the sake of simplicity, we will  consider a model with two communities, yet the extension to more communities is straightforward and would simply involve a linear term in the number of communities. We now describe our model. Once  the $n$ nodes have been split in two groups of size $n\sim n_1, n \sim n_2$, we assume that each edge between two different nodes is sampled independently with probability $p_n$ (or simply $p$ if not ambiguous) if they belong to the same community and $q$ otherwise. We assume that $q= \tau p$ for some constant $\tau\sim \frac{1}{\sqrt n}\ll 1$ and the features belonging to the same community are i.i.d.\@ and $\sigma$-sub-Gaussian, and $\Vert X\Vert\leq 1$. Those assumptions are not restrictive as they hold in many practical applications (and the second, always holds up to a constant). For a given community $i\in\{1,2\}$, we write $(\mu^i_m)_{m\leq N}$ its E-IGT. We impose that $p_n\sim \frac{\log(n)}{n}$ in this particular Bernoulli model. Sparse random graphs do not generally concentrate. Yet, according to \cite{keriven2020sparse}, in the relatively sparse case where $p_n \sim \frac{\log n}{n}$, we get the following spectral concentration bound of the normalized adjacency matrix:

\begin{lemma}
\label{lemma:keriven}
Let $\mathbfcal{A}$ be a symmetric matrix with independent entries $\mathbfcal{A}_{ij}$ obtained as above. If $n_1\sim n, n_2\sim n$, and p is relatively sparse as above, then for all $\nu > 0$, there is a constant $C_\nu$ 
such that, 
with high probability $\geq 1-n^{-\nu}$: 
\begin{equation}
    \Vert \mathbfcal{A}_{\text{norm}} - \mathbb{E}[\mathbfcal{A}]_{\text{norm}} \Vert \leq \frac{C_{\nu}}{\sqrt{\log n}}\,.
\end{equation}
\end{lemma}
\begin{proof}
Can be found in \cite{keriven2020sparse}.
\end{proof}
Note that in general, $\mathbb{E}[\mathbfcal{A}]_{\text{norm}}\neq \mathbb{E}[\mathbfcal{A}_{\text{norm}}]$ and here, because of our model:
\begin{equation}
    \mathbb{E}[\mathbfcal{A}]_{\text{norm}}=\begin{bmatrix}\frac{p}{n_1p+n_2q}\mathbf{1}_{n_1\times n_1}  & \frac{q}{n_1p+n_2q}\mathbf{1}_{n_1\times n_2} \\ \frac{q}{n_1q+n_2p}\mathbf{1}_{n_2\times n_1}  & \frac{p}{n_1q+n_2p}\mathbf{1}_{n_2\times n_2}\end{bmatrix}\,,
\end{equation}
where $\mathbf{1}_{m\times n}$ is a matrix of ones of size $m\times n$. Now, note also that:
\begin{equation}
    \mathbb{E}[\bar U_m]=[\mu^1_m\mathbf{1}_{n_1}^T  ,\mu^2_m  \mathbf{1}_{n_2}^T]\,,
\end{equation}
Now, we prove that the IGT will concentrate around the E-IGT, under a Stochastic Block Model and sub-Gaussianity assumptions. We note that a bias term of the order of $\sqrt{n}\tau$ is present, which is consistent with our model assumptions. Note it is also possible to leverage the boundedness assumption yet it will lead to an additional constant term.
\begin{proposition}\label{prop:concentrate}Under the assumptions above, there exists $C>1$ s.t.\@ for all $N>0, \delta>0$, we have with high probability, larger than $1-\mathcal{O}(N\delta+n^{-\nu})$:
\begin{equation}
    \Vert S_1^NX-\bar S^N X\Vert =\mathcal{O}(\sigma \frac{1+C^N}{1-C}(\sqrt{\ln\frac{1}{\delta}}+\frac{1}{\sqrt{\log n}}))
    +\mathcal{O}(\tau \sqrt{n}\sum_{m\leq N}\Vert \mu^2_m-\mu^1_m\Vert)\,.
\end{equation}
\end{proposition}The following proposition allows to estimate the concentration of each IGT order:
\begin{proposition}
\label{prop:concentrate_each}
Assume that each line of $X\in \mathbb{R}^{n\times P}$ is $\sigma$-sub-Gaussian. There exists $C>1,K>0,C'>1$ such that $\forall m,\delta>0$  with probability $1-8P\delta$, we have: 
\begin{equation}
   \Vert \mathbb{E}[\mathbfcal{A}]_{\text{norm}}\bar U_m-\mathbb{E}[\bar U_m]\Vert 
    \leq K\sigma C^m\sqrt{\ln \frac{1}{\delta}}+C'\sqrt{n}\tau\Vert \mu^2_m-\mu^1_m\Vert\,.
\end{equation}
\end{proposition}This Lemma shows that a cascade of IGT linear isometries preserves sub-Gaussianity:
\begin{lemma}\label{op-subg}
If each line of $X$ is $\sigma$-sub-Gaussian, then each (independent) line of $\bar U_m$ is $C^m\sigma$-sub-Gaussian for some universal constant $C$.
\end{lemma}

In order to show the previous Lemma, we need to demonstrate that the modulus of a sub-Gaussian variable is itself sub-Gaussian, which is shown below:
\begin{lemma}\label{subgausscontractivity}
There is $C>0$, s.t. $X\in\mathbb{R}^P$ is $\sigma$-sub-Gaussian, then $|(X-\mathbb{E}X)W|$ is $C\sigma$-sub-Gaussian.
\end{lemma}


\if False
\subsection{Louis: Graph model and concentration bounds}

From the process to generate $\bar U_m$, note that each row $\bar U^i_m$ for $i\leq n$ is independent. We will use the following Lemma:

\begin{lemma}
If $X\in\mathbb{R}^d$ is $\sigma$-norm-subgaussian (i.e., $\Vert X\Vert$ is subgaussian with parameter $\sigma$), then $|(X-\mathbb{E}X)W|$ is $C\sigma$-norm-subgaussian for some universal constant C.
\end{lemma}
\begin{proof}

\end{proof}
\begin{remark}
$C>1$, as ...
\end{remark}
Now, we will also use the fact that:

\begin{lemma}
If $X_1,...,X_k$ are $\sigma$-norm-subgaussian, then with high probability $1-e^{something}$, we get:
\begin{equation}
\Vert \frac 1n \sum_i X_i-\mathbb{E}X\Vert\leq ?
\end{equation}
\end{lemma}

\begin{proposition}
Combining Lemma \ref{},\ref{},\ref{}, we get:
\end{proposition}

\fbox{Model}: Nodes  intra-  and inter-communities are linked independently with some prescribed probabilities.\\
Sample first the community $C_1,...,C_n$ , and draw an edge between nodes $i,j$ with proba $\mathcal{B}er_{C_i,C_j}$. Hence the (stochastic) adjacency matrix is defined as 
\begin{equation}
    A_{ij} \sim Ber(B_{C_i C_j})
\end{equation}; 
where $B$ is the community probability matrix of size $K*K$ (with $K=$ number of different communities) defined as
\begin{equation*}
B = 
\begin{pmatrix}
p & q & \cdots & q \\
q & p & \cdots & q \\
\vdots  & \vdots  & \ddots & \vdots  \\
q & q & \cdots & p 
\end{pmatrix}
\end{equation*};
where $p=\alpha_n$ and $q= \tau \alpha_n << p$.\\
Each node has some features $X_i,X_j$, independant. We then design the averaging $A$ by renormalizing each row ( we want them to exactly sum to $1$). Then we get $\mathbb{E}[AX|C_1,...,C_n]=\mathbb{E}[A|C_1,...,C_n]\mathbb{E}[X|C_1,...,C_n]$.\\
\begin{lemma}
\label{subgaussian}
Let $X \sim subGaussian(0, \sigma_X^2)$ and $Y \sim subGaussian(0, \sigma_Y^2)$, hence $X+Y \sim subGaussian(0, (\sigma_X+\sigma_Y)^2)$, even if $X$ and $Y$ are not independent.
\begin{proof}
Applying Holder's inequality to the moment generating functions, we have :
\begin{equation}
    \mathbb{E}[|e^{t(X+Y)}|] \leq \mathbb{E}[(e^{tX})^p]^\frac{1}{p} \mathbb{E}[(e^{tY})^q]^\frac{1}{q}
\end{equation}
Taking $p=1+\frac{\sigma_Y}{\sigma_X}$ and $q=1+\frac{\sigma_X}{\sigma_Y}$, we obtain $\mathbb{E}[e^{t(X+Y)}] \leq e^{\frac{(\sigma_X + \sigma_Y)^2t^2}{2}}$.\\
By definition, $X+Y$ is sub-gaussian with parameter $(0, (\sigma_X+\sigma_Y)^2)$.
\end{proof}
\end{lemma}

\begin{proposition}
\label{prop:concentrate}
Suppose each line in our model features follow a subGaussian law that only depends on the community it belongs to ($X_i \sim subGaussian(\mathbb{E}[X_i], (\sigma_{C_i1}, ..., \sigma_{C_ip}))$).
Note $n_l$ the number of node in the community $l$ (among $K$ possible).\\
Thus for $m \geq 0$ with probability $\geq 1-2\sum\limits_{l=1}^K \sum\limits_{j=1}^p e^\frac{-t^2}{2 n_l \sigma_{m l j}^2}$,\\
We have \begin{equation}
    \Vert \mathbb{E}[A] \bar U_m - \mathbb{E}[\bar U_m]\Vert_{MAX} \leq t+\Vert \epsilon \Vert
\end{equation};
where $\Vert \Vert_{MAX}$ is the matrix max norm.
\begin{proof}
First one must point out that line $i$ from the matrix $\mathbb{E}[A] \bar U_m - \mathbb{E}[\bar U_m]$ is the following:
\begin{equation}
    \Vert \frac{1}{n_i} \sum\limits_{k \in C_i} (\bar U_m - \mathbb{E}[\bar U_m])_k + \epsilon_i \Vert
\end{equation}
Where $\epsilon$ is the matrix made of the sum of elements that does not belong to the community of node $i$. The normalization factor of line $i$ is $d_i = p n_i +(n-n_i)q$, and we assume $pn_i >> (n-n_i)q$. As a consequence, the sum of community members (with factor $p$) is leading the global sum with a normalization of $\frac{p}{d_i} \sim \frac{1}{n_i}$, and $\epsilon_i$ is a sum over non-members with a normalization of $\frac{q}{d_i}$ which is neglictable.\\
Lemma \ref{subgaussian} proves that each line (as sum of subGaussian variables) in $\bar U_m$ is also subGaussian. Note we can precisely compute the vector parameter, but the formula become increasingly complex (line $i$ of $\bar U_1$ is subGaussian with parameter $((\sum\limits_{j=1}^p W^{(1)}_{1j}\sigma_{C_ij})^2, ..., (\sum\limits_{j=1}^p W^{(1)}_{pj}\sigma_{C_ij})^2)$; where $W^{(1)}$ is the isometry $W_1$).\\
Hence for a given line $i$, with probability $\leq \sum\limits_{j=1}^p 2 e^\frac{-t^2}{2n_{C_i} \sigma_{mC_ij}^2}$, we have $max \{|\frac{1}{n_i} \sum\limits_{k \in C_i} (\bar U_m - \mathbb{E}[\bar U_m])_k| \} \geq t$.\\
Thanks to our model, the matrix $\mathbb{E}[A] \bar U_m - \mathbb{E}[\bar U_m]$ is made of K blocks in which the $n_l$ lines of each block $l$ are exactlty the same (up to $\epsilon$ matrix) beacause we have same law inside each community. As a consequence, for block $l \in [|1, K|]$ with probability $\leq \sum\limits_{j=1}^p 2 e^\frac{-t^2}{2n_{l} \sigma_{mlj}^2}$, we have $max \{|Block(l)| \} \geq t$.\\ 
Summing it on each block, we have \begin{equation*}
    \Vert \mathbb{E}[A] \bar U_m - \mathbb{E}[\bar U_m]\Vert_{MAX} \geq t
\end{equation*} with probability $\leq \sum\limits_{l=1}^K \sum\limits_{j=1}^p 2 e^\frac{-t^2}{2n_{l} \sigma_{mlj}^2}$.\\
Taking the invert event concludes the proof.
\end{proof}
\end{proposition}

\begin{proposition}Let's note $\alpha_n$ the mean probability of connection (in our model). Make the same assumptions of Proposition \ref{prop:concentrate}.\\
Then, with probability $\geq (1-n^{-t})( 1-2\sum\limits_{l=1}^K \sum\limits_{j=1}^p e^\frac{-t^2}{2 n_l \sigma_{X l j}^2})$, we have the following concentration :
\begin{equation}
    \Vert S_k - \bar S_k \Vert \leq \sqrt{2} \sum\limits_{m=0}^{k-1} \Vert \epsilon_m \Vert + k\sqrt{2}t + 
    \sqrt{\frac{2}{log(n)}} \sum\limits_{m=0}^{k-1} \Vert \bar U_m \Vert
\end{equation}
\begin{proof}
Let's note $\alpha_n$ the mean probability of connection (in the particular Bernouilli model). Sparse random graphs do not concentrate (when the expected degree is very low in comparison to $log (n)$). But according to \cite{spectral_sbm}, in the relatively sparse case where $\alpha_n \sim \frac{log (n)}{n}$, the spectral concentration bound of the normalized adjacency matrix is (with probability $\geq 1-n^{-t}$) : 
\begin{equation}
    \Vert A - \mathbb{E}[A] \Vert \leq \frac{1}{\sqrt{log(n)}}
\end{equation}
As $|W|$ are contractive operator, we have $\Vert \bar U_1 - \mathbb{E}[\bar U_1] \Vert \leq \Vert X - \mathbb{E}[X] \Vert$. Thus concentrating $X$ implies the concentration of other orders of the IGT features.
We apply the inequality (with $m=0$) from Proposition \ref{prop:concentrate} to conclude.
\end{proof}
\end{proposition}

\begin{proposition}
If $W_m$ is unitary, and $A$ is positive, assume that $\mathbb{E}[AX]=\mathbb{E}X$, such that $X$ has variance $\sigma^2$.
\begin{align}
    \mathbb{E}[\Vert S_n-\bar S_n\Vert^2 ] & \leq 2\mathbb{E}[\Vert X\Vert^2 ] -2 \mathbb{E}[\langle AX ; \mathbb{E}[X]\rangle ] \\
    & \leq 2 \sigma^2
\end{align}
\begin{proof}
\begin{align}
    \mathbb{E}[\Vert S_n-\bar S_n\Vert^2 ] & = \mathbb{E}[\sum_{i=0}^n \Vert AU_i - \mathbb{E}[\bar U_i] \Vert^2 ] \\
    & = \mathbb{E}[\sum_i^n \Vert AU_i \Vert^2 + \Vert \mathbb{E}[\bar U_i] \Vert^2 -2\sum_i^n \langle AU_i; \mathbb{E}[\bar U_i] \rangle] \\
    & \leq \mathbb{E}[\sum_i^n \Vert AU_i \Vert^2] + \sum_i^n \Vert \mathbb{E}[\bar U_i] \Vert^2 -2 \mathbb{E}[\langle AX; \mathbb{E}[\bar X] \rangle] \\
\end{align}
Because from $A$ positive, we have $\sum\limits_{i=1}^{n} \langle AU_i; \mathbb{E}[\bar U_i] \rangle \geq 0$. \\
From Equation \eqref{eq:finite_energy}, we have $\sum\limits_{m=0}^{n} \Vert AU_m\Vert^2\leq \Vert U_0\Vert^2$.\\
Similar to Proposition \ref{prop:lip}, $\bar S_n$ is contractive. Hence, $\sum\limits_{i=1}^{n} \Vert \mathbb{E}[\bar U_i] \Vert^2 \leq \mathbb{E}[\Vert \bar U_0\Vert^2] - \mathbb{E}[\Vert \bar U_{n+1} \Vert^2]$. \\
Notifying that $\bar U_0 = U_0 = X$ concludes the proof.
\end{proof}
\end{proposition}
\fi

\subsection{Optimization procedure}
\label{sec:optim}

We now describe the optimization procedure of each of our operators $\{W_n\}$, that consists in a greedy layer-wise procedure \cite{greedy}. Our goal is to specify $|W_n|$ such that it leads to a demodulation effect, as well as to have a fast energy decay. Demodulation means that the envelope of a signal should be smoother, whereas fast decay will allow the use of shallower networks. In practice, it means that at depth $n$, the energy along the direction of averaging should be maximized, which leads to consider:
\begin{equation}
\max_{W^TW=I} \Vert A_J| (I-A_J)U_{n}W|\Vert\,.
\end{equation}
As observed in \cite{oyallon2020interferometric}, because the extremal points of the $\ell^2$ ball are the norm preserving matrix, this optimization problem is equivalent to:
\begin{equation}
\max_{\Vert W\Vert_2\leq 1} \Vert A_J| (I-A_J)U_{n}W|\Vert\,.
\end{equation}
Note that this can be approximatively solved via a projected gradient procedure which projects the operator $W$ on the unit ball for the $\ell^2$-norm at each iteration. Furthermore, contrary to \cite{oyallon2020interferometric}, we might constrain $W$ to have a rank lower than the ambient space, that we denote by $k$: increasing $k$ as well as the order $N$ allows to potentially increase the capacity of our model, yet we as discussed in the next section, this wasn't necessary to obtain accuracies at the level of the state of the art.

\section{Numerical Experiments}
\label{sec:xp}
We test our unsupervised IGT features on a synthetic example, and on challenging semi-supervised  tasks, in various settings that appeared in the graph community labeling litterature: the \textbf{full }  \cite{dropedge}, \textbf{predefined} \cite{kipf2016semi} and \textbf{random splits}  \cite{kipf2016semi} of Cora, Citeseer, Pubmed, as well as the WikiCS dataset.


\subsection{Synthetic example}
\label{sec:synthetic_experiment}
As GCNs progressivly apply a smoothing operator on subsequent layers, deeper features are less sensitive to intra-community variability. This progressive projection can have a big impact on datasets where discriminative features are close in average, yet have very different distributions over several communities. In order to underline this phenomenon, we propose to study the following synthetic example: following the model and notations of Sec. \ref{sec:model}, we consider two communities, with an equal number of samples in each and we assume that $P=1$, $J=1$, $p=0.001$ and $q=0$ and $n=10000$. Here, we assume the features are centered Gaussians with variance $\sigma_1=1$ for the first community and $\sigma_2=\sigma_1+\Delta \sigma$ for the second. In other words, $\Delta\sigma$ controls the relative spread of the community features. Our goal is to show numerically that an IGT representation is, by construction, more amenable to distinguish the two communities than a GCN.
\begin{wrapfigure}{r}{0.5\textwidth}
\begin{center}
\centerline{\includegraphics[width=0.5\columnwidth]{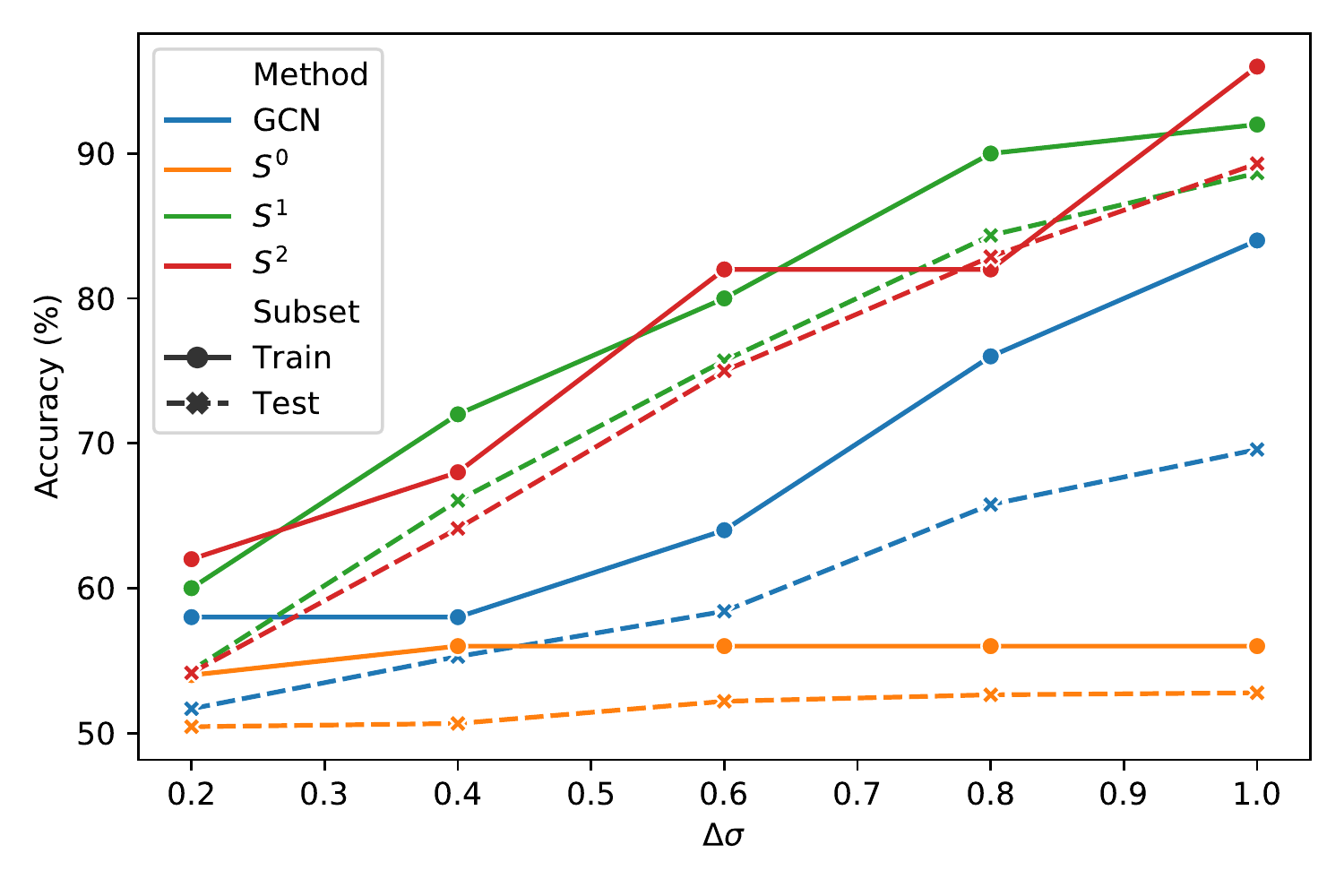}}
\caption{Accuracies of GCN against our  method on a synthetic example, for several values of $\Delta \sigma$.\vspace{-10pt}}\label{fig:toy_plot}\end{center}
\end{wrapfigure}

As a training set, we randomly sample 20 nodes and use the remaining ones as a validation and test set. For IGT parameters, we pick $J=2$, $k=1$ and $N\in\{0,1,2\}$. On top of our standardized IGT features, our classification layer consists in a 1-hidden layer MLP of width 128. We train the IGT operators for 50 epochs. We compare our representation with a standard GCN \cite{kipf2016semi} that has two hidden layers and a hidden dimension of 128 for the sake of comparison. Both supervised architectures are trained during 200 epochs using Adam \cite{kingma2014adam} with a learning rate of 0.01 ($\epsilon=10^{-8},\beta_1=0.9,\beta_2=0.999$). We discuss next the accuracies averaged over 5 different seeds on Fig. \ref{fig:toy_plot} for various representations and values of $\Delta\sigma$.

We observe that an order 0 IGT performs poorly for any values of $\Delta\sigma$, which is consistent because the linear smoothing will dilute important informations for node classification. However, non-linear model like IGT (of order larger than 0) or GCN outperforms this linear representation. The IGT outperforms the GCN for all values of $\Delta\sigma$ because as Sec. \ref{sec:IGT} shows, by construction, this representation extracts explicitly the high frequency of the graph, whereas a GCN can only smooth its features and thus will tend to lose in discriminability despite supervision. We note that orders 1 and 2 perform similarly, which is not surprising given the simplistic assumption of this  model: all the informative variability is contained in the order 1 IGT and the order 2 is likely to only bring noisy features in this specific case.

\subsection{Supervised community detection}
\label{sec:supervised_experiments}
First, we describe our experimental protocol on the datasets Cora, CiteSeer and PubMed. Each dataset consists in a set of bag-of-words vectors with citation links between documents. They are made of respectively 5.4k, 4.7k, 44k and 216k edges with features of size respectively 1.4k, 3.7k, 0.5k and 0.3k. For the three first datasets, we test our method in three semi-supervised settings, which consist in three different approaches to split the dataset into  train, validation and test sets: at this stage, we would like to highlight we are one of the few methods to try its architecture on those three  splits (which we discuss for clarity), which allows to estimate the robustness to various sample complexity. Each accuracy is averaged over 5 runs and we report the standard-deviation in the Appendix. The most standard split is the \textbf{predefined split} setting: each training set is provided by \cite{kipf2016semi} and consist in 20 training nodes per class, which represent a fraction 0.052, 0.036, and 0.003 of the data for Cora, CiteSeer and PubMed respectively. 500 and 1000 nodes are respectively used as a validation and test set. Then, we consider the \textbf{random split} setting introduced in \cite{kipf2016semi}, which is exactly as above except that we randomly extract 5 splits of the data, and we average the accuracies among those splits. Finally, we consider the \textbf{full split} setting which was used in \cite{dropedge} and employs 5 random splits of a larger training set: a fraction 0.45, 0.54 and 0.92 of the whole labeled datasets respectively. Note that each of those tasks is transductive yet our method would require minimal adaptation to fit an inductive pipeline. For WikiCS, we followed the only guideline of \cite{mernyei2020wiki}.

Our architectures are designed as follow: an IGT representation only requires 4 hyper-parameters: an adjacency matrix $\mathbfcal{A}$, an output-size $k$ for each linear isometry, a smoothness parameter $J$ and an IGT order $N$. Given that the graphs are undirected, $\mathbfcal{A}$ satisfies the assumption described in Sec. \ref{sec:IGT}, yet it would be possible to symmetrize the adjacency matrix of a directed graph. This corresponds to our unsupervised graph representation that will be then fed to a supervised classifier. Sec. \ref{sec:model} shows that our IGT representation should concentrate around the E-IGT of their respective community, which means that they should be well separated by a Linear classifier. However, there might be more intra-class variability than the one studied from the lens of our model, thus we decided to use potentially deeper models, e.g., Multi Layer Perceptrons (MLPs) as well as Linear classifiers. We use the same fixed MLP architecture for every dataset: a single hidden layer with 128 features. Our linear model is simply a fully connected layer, and each model is fed to a cross-entropy loss. We note that our MLP is shallow, with few units, and does not involve the graph structure by contrast to semi-supervised GCNs: we thus refer to the combination of IGT and a MLP or a Linear layer as an unsupervised graph representation for node labeling. Note also that a MLP is a scalable classifier in the context of graphs: once the IGT representation is estimated, one can learn the weight of the MLP by splitting the training set in batches, contrary to standard GCNs.

We now describe our training procedure as well as the regularizaton that we incorporated: it was identical for any splits of the data. We optimized our pipeline on Cora and applied it on Citeseer, Pubmed and WikiCS, unless otherwise stated. Each parameter was cross-validated on a validation set, and we report the test accuracy on a test set that was not used until the end.
First, we learn each $\{W_m\}_{m\leq N}$ via Adam for 50 epochs and a learning rate of 0.01. Once computed, the IGT features are normalized and are fed to our supervised classifier, that we train again using Adam and a learning rate of 0.01 for at most 200 epochs, with a early stopping procedure and a patience of 30. A dropout ratio which belongs to $\{0, 0.2, 0.4, 0.5, 0.6, 0.8\}$ is potentially incorporated to the one hidden layer of the MLP. On CiteSeer and PubMed, our procedure selected 0.2, on WikiCS 0.8, whereas no-dropout was added on Cora. Furthermore, we incorporated an $\ell^2$-regularization with our linear layer which we tried amongst $\{0, 0.001, 0.005, 0.01\}$: we picked 0.005 via cross-validation. We discuss here WikiCS: by cross-validation, we used $J=1, N=1, k=150$ for the linear experiment and $J=2, N=1, k=35$ for the MLP experiment. For the other datasets and every splits, we used $N=2$ and $k=10$: we note that less capacity is needed compared to WikiCS, because those datasets are simpler.  For the three other datasets,  for both the \textbf{predefined} and \textbf{random splits}, we fix $J=4$. For the \textbf{full split}, we used $J=1$ for each dataset: we noticed that increasing $J$ degrades the performance, likely because less invariance is required and can be learned from the data, because more samples are available. This makes sense, as the amount of smoothing depends on the variability exhibited by the data. Thanks to the amount of available data, the supervised classifier can estimate the degree of invariance needed for the classification task, which was not possible if using only 20 samples per community.

Tab. \ref{table:semi} reports the semi-supervised accuracy for each dataset, in various settings, and compares standard supervised~\cite{gao2019graph, kipf2016semi,dropedge,chen2018fastgcn,chen2020measuring,velivckovic2017graph} and unsupervised \cite{perozzi2014deepwalk,velickovic2019deep,garcia2017learning,qu2019gmnn,hamilton2017inductive} architectures. Note that each supervised model is trained in an end-to-end manner. The unsupervised models are built differently and we discuss them now briefly: for instance,  EP~\cite{garcia2017learning}, uses a node embedding with a rather different architecture from GCNs. Also, DeepWalk~\cite{perozzi2014deepwalk} is analogous to a random walk, GraphSage~\cite{hamilton2017inductive} learns an embedding with a local criterion, DGI~\cite{velickovic2019deep} relies on a mutual information criterion and finally \cite{qu2019gmnn} relies on a random field model. Note that each of those models are significantly different from ours and they do not have the same  theoretical foundations and properties as ours. As expected, accuracy in the \textbf{full} setting is higher than the others. We observe that in general, supervised models outperforms  unsupervised models by a large margin except on WikiCS and Citeseer for the \textbf{random} and \textbf{predefined} splits, for which an IGT obtains better accuracy: it indicates that it has a better inductive bias for this dataset. Note that an IGT obtains competitive accuracies amongst unsupervised representations and this is consistent with the fact that those datasets, discussed above, are likely to satisfy the hypothesis described in Sec. \ref{sec:model}. In general, a MLP outperforms a linear layer (because it has better approximation properties), except on Citeseer for which the accuracy is similar, which seems to validate that the data of Citeseer follow the model that we introduced in \ref{sec:model} on Citeseer, that leads to linear separability. 

\if False
\begin{table}[t]
\caption{Classification accuracies for the full supervised settings on Cora, Citeseer and Pubmed datasets.}
\label{table:full}
\begin{center}
\begin{small}
\begin{sc}
\begin{tabular}{lcccc}
\toprule
Method & Cora & Citeseer & Pubmed & Data augmentation\\
\midrule
GCN \cite{dropedge}  & 86.6 & 79.3 & 70.7 & $\times$ \\
Fastgcn \cite{fast_gcn} & 86.5 & - & 88.8 & $\times$\\
DropEdge \cite{dropedge}    & 88.2 & 80.5 & 91.7 & $\surd$ \\
\cite{adapt}   & 87.4 & 79.7 &  90.6 & $\surd$ \\
\midrule
MLP on $S_J^N$  (ours)   & 0 & 0 & 0 & $\times$\\
\bottomrule
\end{tabular}
\end{sc}
\end{small}
\end{center}
\vskip -0.1in
\end{table}
\fi

\begin{table}[t]
\caption{Classification accuracies (in \%) for each splits of Cora, Citeseer, Pubmed as well as WikiCS.}
\vspace{-4pt}
\label{table:semi}
\begin{center}
\begin{small}
\begin{tabular}{lcccccccccc}
\toprule
Method/Dataset & \multicolumn{3}{c}{Cora}&\multicolumn{3}{c}{Citeseer}&\multicolumn{3}{c}{Pubmed}&WikiCS\\

& Full& Rand& Pred& Full& Rand& Pred& Full& Rand& Pred\\
\midrule
Supervised&\\
\midrule
GAT \cite{velivckovic2017graph}& &&83.0     & & & 72.5 & & & 79.0&77.2 \\
GCN \cite{kipf2016semi}& &80.1&81.5 & & 67.9&  70.3 & &\textbf{78.9} & 79.0& \textbf{77.7}\\
Graph U-Net \cite{gao2019graph} & &&\textbf{84.4}&&&\textbf{73.2}&&&\textbf{79.6}&\\
DropEdge \cite{dropedge} &  \textbf{88.2}&& & \textbf{80.5} &&& \textbf{91.7}&&&\\
FastGCN \cite{chen2018fastgcn}& 85.0&& & 77.6 &&& 88.0 &&&\\
OS \cite{chen2020measuring}& &\textbf{82.3} &&& \textbf{69.7}&&& 77.4& \\
\midrule
Unsupervised\\
\midrule
Raw \cite{velickovic2019deep,mernyei2020wiki}& &&47.9&&&49.3&&&69.1&72.0\\
DeepWalk \cite{perozzi2014deepwalk}& &&67.2&& & 43.2&& & 65.3& 74.4\\
IGT + MLP (ours)  &\textbf{87.7} &\textbf{78.3}&80.3&\textbf{78.4} &67.6& \textbf{73.1 }&\textbf{88.2}&76.2&76.4& \textbf{77.2} \\
IGT + Lin.  (ours)  & 83.3 &77.6&77.4& \textbf{78.4} &\textbf{73.0}&\textbf{73.1} &88.1&74.5& 73.9&76.7\\
EP \cite{garcia2017learning}&&78.1&&&71.0&&&\textbf{79.6}&\\
GraphSage \cite{hamilton2017inductive}&82.2&&&71.4&&&87.1&&&\\
DGI \cite{velickovic2019deep}&&&82.3&&&71.8&&&76.8&75.4\\
GMNN \cite{qu2019gmnn}&&&\textbf{82.8}&&&71.5&&&\textbf{81.6}&\\
\bottomrule
\end{tabular}
\end{small}
\vspace{-20pt}
\end{center}
\end{table}

\subsection{Ablation experiments}
\label{sec:ablation}\begin{wraptable}{r}{0.5\textwidth}
\vspace{-15pt}
\caption{Linear classification accuracies (in \%) for the \textbf{predefined split} on Cora's \textit{validation set}, for various values of $N,J$.\label{table:linear_classif}}
\begin{center}
\begin{small}
\begin{sc}
\begin{tabular}{c|cccc}
\toprule
\backslashbox{J}{N} & 0 & 1 & 2 & 3\\
\midrule
1 & 62.4 & 60.8 & 62.8 & 61.4\\
2 & 68.6 & 70.6& 72.2& 68.6 \\
3 & 71.4& 72.2& \textbf{74.6}& 72.6 \\
4 & 72.4& 73.2& \textbf{74.6}& 73.0\\
\bottomrule
\end{tabular}
\end{sc}
\vspace{-10pt}
\end{small}
\end{center}
\end{wraptable}
In order to understand better the IGT representation, we propose to study the accuracy of an IGT representation on Cora's validation set, as a function of the scale $J$ and the IGT order $N$. For the sake of simplicity, we consider a linear classifier. Each linear operator is learned with 40 epochs. We picked $k=10$ and train our basic model for 200 epochs with SGD, the validation accuracies are reported in Tab. \ref{table:linear_classif}. As  $N,J$ increase, we feed the features to a linear classifier: in general, for $0\leq N\leq 2$, as $N$ grows the accuracy improves. However, the order 3 IGT decreases the accuracy: this is consistent because it conveys a noise which is amplified by the standardization.
As $J$ increases, we smooth our IGT features on more neighbor nodes, which results in better performances for a fixed order $N$, and is also consistent with the finding of Sec. \ref{sec:IGT}.

We performed a second ablation experiment in order to test the inductive bias of our architecture: we considered random  $\{W_n\}$ at evaluation time and we obtained respectively on the full split of Cora, Citeseer and Pubmed some accuracy drops of respectively 6.3\%, 5.2\% and 5.6\%. This is relatively smaller drops than DGI~\cite{velickovic2019deep} which reports for instance some drops of about $10\%$: our architecture is likely to have a better inductive bias for this task.

\vspace{-5pt}

\section{Conclusion}
\vspace{-5pt}
In this work, we introduced the IGT which is an unsupervised and semi-supervised representation for community labeling. It consists in a cascade of linear isometries and point-wise absolute values. This representation is similar to a semi-supervised GCN, yet it is trained layer-wise, without using labels, and has strong theoretical fundations. Indeed, under a SBM assumption and a large graph hypothesis, we show that an IGT representation can discriminate  communities of a graph from a single realization of this graph. It is numerically supported by a synthetic example based on Gaussian features, which shows that an IGT can estimate the community of a given node better than a GCN because it tends to alleviate the over-smoothing phenomenon. This is further supported by our numerical experiments on the standard, challenging datasets Cora, CiteSeer, PubMed and WikiCS: with shallow supervised classifiers, we obtain numerical accuracy which is competitive with semi-supervised approaches.

Future directions could be to either refine our theoretical analysis by weakening our assumptions, or to test our method on inductive tasks. Furthermore, following \cite{oyallon2020interferometric}, one can also wonder if this type of approach could be extended to more complex data, in order to obtain stronger theoretical guarantees (e.g., manifold).
Finally, future works could also be dedicated to scale our algorithms to very large graphs: this is a challenging task both in terms of memory and computations.
\paragraph{Broader impact.} Graph Neural Networks can be used in many domains, like protein prediction, or network analysis to cite only a few, and could become even more prevalent tomorrow. Our work is thus included in a large literature whose societal impact and ethical considerations are to become more and more important. We provide here a new model aiming at learning unsupervised node representation in community graphs graphs. While its most natural application lies in community detection in social science, we hope that the provided theoretical guarantees could be used in the future to provide safer and more readable models toward more various directions.

\section*{Acknowledgements}
EO, LL, NG would like to acknowledge the CIRM for its one week hospitality which was helpful to this project.  EO acknowledges NVIDIA for its GPU donation and this work was granted access to the HPC resources of IDRIS under the allocation 2020-[AD011011216R1] made by GENCI. EO, LL were partly supported by ANR-19-CHIA "SCAI" and  ANR-20-CHIA-0022-01 "VISA-DEEP". NG and PP would like to acknowledge the support of the French Ministry of Higher Education and Research, Inria, and the Hauts-de-France region; they also want to thank the Scool research group for providing a great research environment.
The authors would like to thank Mathieu Andreux, Alberto Bietti, Edouard Leurent, Nicolas Keriven, Ahmed Mazari, Aladin Virmaux for helpful comments and suggestions.

\bibliography{citations}
\bibliographystyle{plain}

\newpage
\section{Appendix}
\subsection{Proofs}

\setcounter{lemma}{1}

\begin{lemma}\label{proj}
If $0 \preccurlyeq A \preccurlyeq I$, then $\Vert AX\Vert^2+\Vert (I-A)X\Vert^2\leq \Vert X\Vert^2$ with equality iff $A^2=A$.
\end{lemma}
\begin{proof}We note that for any $x$, we get: 
\begin{equation}
    \Vert Ax\Vert^2+\Vert (I-A)x\Vert^2=\Vert Ax\Vert^2+\Vert x\Vert^2+\Vert Ax\Vert^2-2\langle x,Ax\rangle
\end{equation}

Yet, $\Vert Ax\Vert^2=\langle x,A^TAx\rangle \leq \langle x,Ax\rangle$ because $\text{Sp}(A)\subset[0,1]$. Thus, 
\begin{equation}
    2(\Vert Ax\Vert^2-\langle x,Ax\rangle)+\Vert x\Vert^2\leq \Vert x\Vert^2\,,
\end{equation}
with equality $\forall x$ iff $A=A^2$. It is now enough to observe that $\{A,I-A\}$ inherits from those properties.
\end{proof}

The following proposition explains that our representation is non-expansive, and thus stable to noise:

\begin{proposition}
\label{prop:lip}
For $N\in\mathbb{N}$, $S_J^NX$ is 1-Lipschitz leading to:
\begin{equation}
\Vert S_J^NX-S_J^NY\Vert\leq\Vert X-Y\Vert\,.
\end{equation}
and furthermore:
\begin{equation}
\Vert S_J^NX\Vert\leq\Vert X\Vert\,.
\end{equation}
\end{proposition}
\begin{proof}
For two feature matrices $X,Y$, let us consider $U_{i}$ and $\tilde U_{i}$ defined from  Equation \eqref{eq:u}, with $U_{0}=X$ and $\tilde U_{0}=Y$.  Because $|W_i|$ is a contractive and from Lemma \ref{proj},
\begin{align}
    \Vert U_{i+1} - \tilde U_{i+1}\Vert^2 & \leq  \Vert U_{i} - \tilde U_{i}- A_J ( U_{i} -\tilde U_{i})\Vert^2 \\
    &\leq\Vert U_{i} - \tilde U_{i} \Vert^2 - \Vert A_J ( U_{i} - \tilde U_{i})\Vert^2
\end{align}
Hence,
\begin{align}
    \sum_i^N \Vert A_J ( U_{i} - \tilde U_{i})\Vert^2 & \leq  \Vert X-Y\Vert^2 - \Vert U_{n} - \tilde U_{n}\Vert^2 \\
    & \leq \Vert X-Y\Vert^2
\end{align}
Taking $X=0$ leads to the second part as then $SX=0$.
\end{proof}

\setcounter{lemma}{4}

This Lemma shows that a cascade of IGT linear isometries preserve sub-Gaussianity:
\begin{lemma}\label{op-subg}
If each line of $X$ is $\sigma$-sub-Gaussian, then each (independent) line of $\bar U_m$ is $C^m\sigma$-sub-Gaussian for some universal constant $C$.
\end{lemma}
\begin{proof}
Apply the Lemma \ref{subgausscontractivity} with $W=W_n$ for $n\leq m$ leads to the result.
\end{proof}
\setcounter{lemma}{5}
In order to show the previous Lemma, we need to demonstrate that the modulus of a sub-Gaussian variable is itself sub-Gaussian, which is shown below:
\begin{lemma}\label{subgausscontractivity}
If $X\in\mathbb{R}^P$ is $\sigma$-sub-Gaussian, then $|(X-\mathbb{E}X)W|$ is $C\sigma$-sub-Gaussian for some absolute value $C$.
\end{lemma}
\begin{proof}
If $X$ is $\sigma$-sub-Gaussian, then $X-\mathbb{E}X$ is $C'\sigma$-subGaussian by recentering \cite{high_dim_proba}. We note that as $W$ is unitary, thus $(X-\mathbb{E}X)W$ is also $C'\sigma$-subgaussian. Then, let $u\in\mathbb{R}^p$ an unit vector. We note that:
\begin{align}
    &\mathbb{P}(\sum_{i=1}^pu_i|X_i|\geq t)\\
    &\leq \sum_{\epsilon_i\in\{-1,1\}} \mathbb{P}(\{\epsilon_iX_i\geq 0\}\cap\{\sum_iu_i|X_i|\geq t\})\\
    &=\sum_{\epsilon_i\in\{-1,1\}} \mathbb{P}(\{\epsilon_iX_i\geq 0\}\cap\{\sum_i\epsilon_iu_iX_i\geq t\})\\
    &\leq 2^p e^{-\frac{t^2}{2C'^2\sigma^2}}=e^{p\ln 2-\frac{t^2}{2C'^2\sigma^2}}\,.
\end{align}
This leads to the conclusion by sub-Gaussian characterization.
\end{proof}

\setcounter{proposition}{2}

\begin{proposition}\label{prop:boundigt}
For any $X, N, J$, we get:
\begin{equation}
  \Vert S_J^NX-\bar S^N X\Vert\leq   \sqrt{2} \sum_{m=0}^{N}\Vert  (A_J-\mathbb{E})\bar U_m\Vert \,.
\end{equation}
\end{proposition}
\begin{proof}
Here, write $V_J^m=|(X-A_JX)W_m|, V_J^0 X=X$, and define: \begin{align}
    Y_J^{n,m} X=&\{A_JV_J^n...V_J^{n-m+1}X,A_JV_J^{n-1}...V_J^{n-m+1}X\\
    &,...,A_JV_J^{n-m+1}X,A_JX\},
\end{align}

\paragraph{Lemma.} \textit{If $A_J$ is a unitary projector and each $W_i$ is unitary, then $Y_J^mX$ is 1-Lipschitz w.r.t. $X$.}
\begin{proof}
We can apply the proposition 2 with the operators $\{W_{n-m+1},...,W_n\}$, as this can be interprated as an IGT with different unitary operators.
\end{proof}

Here the idea is to take advantage of the tree structure of the IGT features. Thus when $Y_J^{n,m}$ is computing $S_J^n$ to orders limited in $[n, n-m+1]$, we chain the features with the order $n-m$ to recover $Y_J^{n,m-1}$. To do so, we introduce for $m\geq 1$ :
\begin{align}
\Delta_J^{n,m} X&=\{Y_J^{m}V_J^{n-m}X-Y_J^{m}\bar V^{n-m}X,A_JX-\mathbb{E}X\}\\
&=\{-Y_J^{m}\bar V^{n-m}X,-\mathbb{E}X\}+Y_J^{m-1}X\,,\label{eq:delta}
\end{align}
where $\bar V^n X=|(X-\mathbb{E}X)W_n|$, $\bar V^0X=X$ and $\{x,y\}$ stands for a concatenation. This implies that $\Delta_J^{n,m} X$ is a $(m+1)$-uplet (and the symbol $+$  in (\ref{eq:delta}) is thus a couple addition and the convention is that left corresponds to highest order of the couple), and $\Delta_J^{0,0}X=A_JX-\mathbb{E}X=-\mathbb{E}X+S_J^0X$.\\
The sum over $m$-uplet with different size is done such that the left elements are summed first. We then notice that:
\begin{equation}
    \sum_{m=0}^{N} \Delta_J^{N,m}\bar V^{N-m-1}...\bar V_1X=S_J^N X-\bar S^N X\,
\end{equation}
because each term of the couple is a telescopic sum (again here, we chain the features with orders in $[n-m-1, 1]$ to obtain the telescopy).\\
As $Y_J^{n,m}$ is 1-Lipschitz w.r.t. $X$ and since a modulus is non expansive, $\Vert |(X-A_JX)W_n|-|(X-\mathbb{E}X)W_n|\Vert\leq \Vert \mathbb{E}X-A_JX\Vert$, combining those ingredients we get:
\begin{align}
    \Vert \Delta_J^{n,m} X\Vert^2 = & \Vert A_JX-\mathbb{E}X\Vert^2+\\
    & \Vert Y_J^{m-1}|(X-A_JX)W_n|-Y_J^{m-1}|(X-\mathbb{E}X)W_n\Vert^2\\
    & \leq 2 \Vert A_J X-\mathbb{E}X\Vert^2\,.
\end{align}

Then, we further apply the triangular inequality to get the desired result.
\end{proof}
The following proposition shows that in case of exact ergodicity, the IGT and Expected-IGT representations have bounded moments of order 2:
\begin{proposition}
Assume that $\mathbb{E}[A_JX]=\mathbb{E}X$, and that $X$ has variance $\sigma^2=\mathbb{E}\Vert X \Vert^2- \Vert \mathbb{E}X \Vert^2$, then:
\begin{align}
    \mathbb{E}[\Vert S_J^NX-\bar S^NX\Vert^2 ] & \leq  2 \sigma^2\,.
\end{align}
\begin{proof}
\begin{align}
   &  \mathbb{E}  [\Vert S_J^NX-\bar S^NX \Vert^2 ]=\mathbb{E}[\Vert S_J^NX\Vert^2]+\mathbb{E}[\Vert \bar S^NX\Vert^2]\\
    &-2\sum_{m=0}^N \mathbb{E}[\text{Tr}( (A_J U_m)^T \mathbb{E}[\bar U_m] )]\\
    &\leq  2(\mathbb{E}\Vert X\Vert^2-\sum_{m=0}^N \mathbb{E}[\text{Tr}( (A_J U_m)^T \mathbb{E}[\bar U_m] )])
\end{align}
The inequality follows from Prop. \ref{proj} and Prop. \ref{prop:lip}. Now, from Lemma \ref{lemma:positive}, $A_J, U_m, \bar U_m$ have positive coefficients, thus we get: $2\sum_{m=1}^N \mathbb{E}[\text{Tr}( (A_J U_m)^T \mathbb{E}[\bar U_m] )] \geq 0$. The first term allows to conclude as $\bar U_0=U_0=X$.
\end{proof}
\end{proposition}
\setcounter{proposition}{5}
\begin{proposition}
\label{prop:concentrate_each}
Assume that each line of $X$ is $\sigma$-sub-Gaussian.
There exists $C>1,K>0,C'>0$ such that $\forall m,\delta>0$  with probability $1-8P\delta$, we have: \begin{align}
   \Vert \mathbb{E}&[\mathbfcal{A}]_{\text{norm}}\bar U_m-\mathbb{E}[\bar U_m]\Vert \\
   & \leq K\sigma C^m\sqrt{\ln \frac{1}{\delta}}+\tau\sqrt{n}C'\Vert \mu^2_m-\mu^1_m\Vert\,.
\end{align}
\end{proposition}
\begin{proof}
Here, for the sake of simplicity, $X_p$ corresponds to the $p$-th row of $X$. We write $\mu_m^j$ the expected-IGT of the node distribution of community $j$. Here, we have for $t\leq n_1$ (note that the right does not depend on $t$):
\begin{align}
   &[\mathbb{E}[\mathbfcal{A}]_{\text{norm}}\bar U_m]_t-\mathbb{E}[\bar U_m]_t\\
   &=\frac{1}{n_1p+n_2q}\big(p\sum_{i=1}^{n_1}(\bar U^i_m-\mu_m^1)+q\sum_{i=n_1+1}^{n_1+n_2}(\bar U^i_m-\mu_m^2)\big)\\
   &+\frac{n_2q}{n_1p+n_2q}(\mu^2_m-\mu^1_m)\,.
\end{align}
Now, we note that from Lemma \ref{op-subg}, $\{\bar U^i_m\}_{i\leq n}$ is a family of $\sigma C^m$-sub-Gaussian independant r.v. From Hoeffding lemma \cite{high_dim_stat}, we obtain that for any $\delta$, we have with probability $1-4P\delta$:
\begin{align*}
\Vert\sum_{i=1}^{n_1}(\bar U^i_m-\mu^1_m)\Vert\leq \sqrt{n_1}\sqrt{2}\sigma C^m\sqrt{\ln\frac{1}{\delta}} \,\,\,\text{  and  }\\
\Vert\sum_{i=n_1+1}^{n_1+n_2}(\bar U^i_m-\mu^2_m)\Vert\leq \sqrt{n_2}\sqrt{2}\sigma C^m\sqrt{\ln\frac{1}{\delta}}\,.
\end{align*}
As if $n$ is large, by hypothesis $(\frac{p\sqrt{2n_1}+q\sqrt{2n_2}}{n_1p+n_2q})\sqrt{n}=\mathcal{O}(1)$. We perform the same for $n_1<t\leq n_1+n_2$ We then sum along $n$ and use that $\frac{n_1}{n_2+\tau n_1}+\frac{n_2}{n_1+\tau n_2}=\mathcal{O}(1)$ and $\sqrt{a+b}\leq\sqrt{a}+\sqrt{b}$.\end{proof}

\subsection{Dataset statistics}

\begin{table}[th]
\label{dataset_statistics}
\begin{center}
\caption{Dataset Statistics}
\begin{tabular}{ccccccc}
\bf Datasets & \bf Nodes & \bf Edges & \bf Classes & \bf Features & \textit{full} \bf Train/Val/Test & \textit{semi} \bf Train/Val/Test
\\ \hline \\
Cora & 2,708 & 5,429 & 7 & 1,433 & 1,208/500/1,000 & 140/500/1,000 \\
Citeseer & 3,327 & 4,732 & 6 & 3,703 & 1,812/500/1,000 & 120/500/1,000 \\
Pubmed & 19,717 & 44,338 & 3 & 500 & 18,217/500/1,000 & 60/500/1,000 \\
WikiCS & 11,701 & 216,123 & 10 & 300  &\multicolumn{2}{c}{20 canonical train/valid/test splits}\\
\end{tabular}
\end{center}
\end{table}
\newpage

\begin{table}[t]
\caption{Standard deviations of classification accuracies for each splits of Cora, Citeseer, Pubmed as well as WikiCS.}
\label{table:semi}
\begin{center}
\begin{small}
\begin{tabular}{lcccccccccc}
\toprule
Method/Dataset & \multicolumn{3}{c}{Cora}&\multicolumn{3}{c}{Citeseer}&\multicolumn{3}{c}{Pubmed}&WikiCS\\

& Full& Rand& Pred& Full& Rand& Pred& Full& Rand& Pred\\
\midrule
Unsupervised\\
\midrule

IGT + MLP (ours)  & 0.5 & 0.8 & 0.9 & 0.4 & 0.8 & 0.7 & 0.6 & 0.5 & 0.3 & 0.5 \\
IGT + Lin.  (ours)  & 0.1 & 0.8 & 0.2 & 0.3 & 0.7 & 0.5 & 0.1 & 0.2 & 0.1 & 0.5\\

\bottomrule
\end{tabular}
\end{small}
\end{center}

\end{table}

\subsection{Code and Data availability}
All the code is accessible in the folder given in the supplementary materials.

\subsection{Training time}
We informally noticed that the training of our isometry layers converges quickly. During the supervised training, no multiplication with the adjacency matrix is involved, which can speed up the training compared to GCNs.
We further report wall-clock training time in seconds until convergence for our method and for GCNs. For the latter, we used an implementation provided by the authors and trained on the same hardware (with GPU) as our IGT model.
For Cora, Citeseer and PubMed respectively, the training time of our IGT layers was 0.45s, 0.57s and 4.88s, whereas the training time of the classification head was 0.25s, 0.24s and 0.94s.
By way of comparison, GCN training time was 0.86s, 1.82s, and 1.12s. We would like to highlight that our code works on limited resources and we used a total of 10 GPU hours  for developing and benchmarking this project.
\end{document}